%% file: ms.tex
\documentclass[11pt]{article} 

\usepackage{graphicx}
\usepackage{booktabs}
\usepackage{algorithm}
\usepackage{algpseudocode}
\usepackage[T1]{fontenc}
\usepackage{lmodern}
\usepackage[flushleft]{threeparttable}
\usepackage[hidelinks]{hyperref} %
\hypersetup{hidelinks}
\usepackage{multirow}
\usepackage{longtable}
\usepackage[margin=3cm]{geometry} %
\usepackage{amsmath,amsthm,amssymb}
\usepackage{authblk}
\usepackage{caption}
\usepackage{microtype}

\newtheorem{theorem}{Theorem}
\newtheorem{lemma}{Lemma}

\DeclareMathOperator*{\argmax}{arg\,max}
\DeclareMathOperator{\cossim}{cossim}

\title{Using predefined vector systems to speed up neural network
multimillion class classification}
\author[1]{Nikita Gabdullin}
\author[1]{Ilya Androsov}
\affil[1]{Joint Stock "Research and production company "Kryptonite" \authorcr
E-mail: n.gabdullin@kryptonite.ru, i.androsov@kryptonite.ru}
\date{}
\begin{document}

    \captionsetup[table]{labelformat={default},labelsep=period,name={Table}}

    \maketitle

    \begin{abstract}
        Label prediction in neural networks (NNs) has
        $\mathcal{O}(n)$ complexity proportional to the number of classes. This holds
        true for classification using fully connected layers and cosine
        similarity with some set of class prototypes. In this paper we show that
        if NN latent space (LS) geometry is known and possesses specific
        properties, label prediction complexity can be significantly reduced.
        This is achieved by associating label prediction with the $\mathcal{O}(1)$
        complexity closest cluster center search in a vector system used as
        target for latent space configuration (LSC). The proposed method only
        requires finding indexes of several largest and lowest values in the
        embedding vector making it extremely computationally efficient. We show
        that the proposed method does not change NN training accuracy
        computational results. We also measure the time required by different
        computational stages of NN inference and label prediction on multiple
        datasets. The experiments show that the proposed method allows to
        achieve up to 11.6 times overall acceleration over conventional methods.
        Furthermore, the proposed method has unique properties which allow to
        predict the existence of new classes.
  
    \end{abstract}

    
    \input{full}

\end{document}

%% file: full.tex
\section{Introduction}
\label{introduction}

Conventional neural networks (NNs) used in computer vision (CV) tasks
are trained to distinguish multitudes of specific objects or their
categories. Furthermore, it is becoming essential for NNs to be able to
correctly identify different object categories (like animals or
vehicles) and their subcategories (like cat breeds or vehicle types) at
the same time. This results in extremely large number of classes in
classification tasks, which can reach tens and hundreds of millions.
This topic is relevant for various application areas which include
retail, image and video analysis systems, recommendation systems, etc.

Class or label prediction is achieved by performing some operation on NN
outputs called representations or embeddings, with the space of all
possible embeddings referred to as latent space (LS). Typically, NNs
learn representations of input data with no external guidance with
respect to the resulting embedding cluster distribution in LS and the
overall LS geometric structure. However, in recent years many researches
have been focusing on identifying LS properties desired for high NN
performance and ways to ensure that they are obtained during 
training~\cite{CosFace,CenterLoss}. It has recently been shown that one can use
predefined system of vectors to train NNs with a method called latent
space configuration (LSC) to obtain LS with known structure and the
desired set of properties~\cite{lsc2}.

It was previously noted that the knowledge of the structure of vector
system used for LSC allows to achieve $\mathcal{O}(1)$ complexity for closest
center vector search operation, which is necessary to predict labels~\cite{Igor}.
In this paper we experimentally verify this for specific
family of configurations described in Section~\ref{Vn}. This allows to
achieve the astounding (up to 11.6 times in our experiments) speed up
over conventional inference process when the number of classes
(\emph{n\textsubscript{classes}}) is huge and GPU memory budget is
limited. The proposed method also has unique properties which
potentially can assist in tasks which require new class prediction.

The rest of the paper is organized as follows: Section~\ref{cc} outlines the
methodology of the new search method, Section~\ref{Exp} provides its
experimental verification, Section~\ref{discussions} discusses possibility of new class
prediction, and Section~\ref{conclusions} concludes the paper.

\section{Closest center vector search in preconfigured LS}
\label{cc}

\subsection{Vector systems for latent space configurations}
\label{Vn}

Conventionally NNs develop embeddings during training along with their
distribution in LS with number of dimensions
\emph{n\textsubscript{dim}}. However, researchers noticed that
properties of embedding distribution significantly affect the NN
performance~\cite{CosFace,CenterLoss}. They proposed different loss
functions that improve the separation of clusters of embeddings.
However, these methods do not guarantee that specific LS geometry will
be reached during training. It has recently been shown that this can be
achieved with LSC~\cite{lsc2, lsc1}: vector systems which have desired
properties can be used as configuration targets to ensure that LS has
these properties, too~\cite{lsc3}.

In this study we focus on specific family of vector systems with base
vectors consisting only of ones, negative ones, and zeroes. Such systems
are referred to as \(V_{n}^{D}\) with the first digit in \emph{D} being
the number of ones, the second being the number of negative ones, and
\emph{n} being the LS dimension (so
\emph{n}=\emph{n\textsubscript{dim}}). The choice of vector system index
\emph{D} and the dimension \emph{n} determines the number of vectors
\emph{n\textsubscript{vects }}in that system~\cite{lsc3}. Since these
vectors act as center vectors for clusters of class embeddings during
training, \emph{D} and \emph{n} determine the maximum possible number of
classes that NN can be trained to distinguish. Therefore, when
\emph{n\textsubscript{classes }}is known in advance,
\emph{n\textsubscript{dim}} should be chosen so that
\emph{n\textsubscript{vects}} \(\geq\) \emph{n\textsubscript{classes}}.
In case of \(V_{n}^{22}\) the number of vectors is

\begin{equation}
	n_{vects} = \frac{n (n - 1) (n - 2) (n - 3)}{4},
	\label{eq:n_vects}
\end{equation}
\unskip

where in context of NN LSC training \emph{n} =
\emph{n\textsubscript{dim}}. Hence, the maximum possible number of
classes \emph{n\textsubscript{classes}} = \emph{n\textsubscript{vects}}
\textasciitilde{} \emph{n\textsubscript{dim}}\textsuperscript{4}, or
\emph{n\textsubscript{dim}} \textasciitilde{} \(\sqrt[4]{n_{classes}}\).
This allows to accommodate a huge number of well-separated vectors in LS
of conventional NN models. It has also been shown in~\cite{lsc3} that
training with the minimum possible \emph{n\textsubscript{dim}} required
to accommodate the desired number of classes yields the best results.
This criterion is used to choose \emph{n\textsubscript{dim}} with (\ref{eq:n_vects})
in all experiments in Sections~\ref{Exp}.

\subsection{Vector system closest center search}
\label{fast-search}

\begin{algorithm}
	\caption{Finding the closest center in \(V_{n}^{mk}\) and label prediction (PyTorch-style pseudocode)}
	\label{alg:1}
	\begin{algorithmic}[1]
		\Require NN model, batch of inputs $x$, embeddings $z = \text{model}(x)$, 
		          center vectors $C$ of vector system $V_{n}^{mk}$, 
		          label dictionary $\{\text{labels}_i : C_i\}$ with $i = 0 \ldots n_{\text{classes}} - 1$
		
		\Statex \textbf{Preliminary step}
		\State $c_{\text{maxes}} \gets \text{sort}(\text{topk}(C, m, \text{largest=True}))$ \Comment{find indexes of $m$ largest values for vectors in $C$}
		\State $c_{\text{mins}} \gets \text{sort}(\text{topk}(C, k, \text{largest=False}))$ \Comment{find indexes of $k$ lowest values for vectors in $C$}
		\State $c_{\text{sorted}} \gets (c_{\text{maxes}}, c_{\text{mins}})$ \Comment{make a tuple of sorted maxes and sorted mins}
		\State $\text{center\_dict} \gets \{c_{\text{sorted}} : \text{labels}\}$ \Comment{replace center vectors $C_i$ in label dictionary with $c_{\text{sorted}}$ and rearrange}
		
		\Statex \textbf{Inference step}
		\State $z_{\text{maxes}} \gets \text{sort}(\text{topk}(z, m, \text{largest=True}))$ \Comment{find indexes of $m$ largest values for vectors in $z$}
		\State $z_{\text{mins}} \gets \text{sort}(\text{topk}(z, k, \text{largest=False}))$ \Comment{find indexes of $k$ lowest values for vectors in $z$}
		\State $z_{\text{idxs}} \gets (z_{\text{maxes}}, z_{\text{mins}})$ \Comment{make tuple of sorted maxes and sorted mins}
		\State $\text{predicted\_labels} \gets \text{center\_dict}[z_{\text{idxs}}] \text{ if } z_{\text{idxs}} \in \text{center\_dict} \text{ else } -1$ \Comment{find label using sorted indexes; if closest center not in center dict return -1}
		\State \Comment{(optional) reconstruct the closest center vector using $z_{\text{idxs}}$ or $C[\text{predicted\_labels}]$}
	\end{algorithmic}
\end{algorithm}

In conventional classifiers the predicted label is determined by the
output of the final fully connected layer. This layer performs matrix
multiplication of NN embeddings with weight matrix which size depends on
\emph{n\textsubscript{classes}}. This can become a problem when
\emph{n\textsubscript{classes}} is huge since NN inference is the
fastest if the entire weight matrix is stored on GPU. Alternatively, a
similarity search with some set of precomputed or predefined prototype
embedding vectors can be used~\cite{proto}. However, the complexity of
such search is still proportional to \emph{n\textsubscript{classes
}} which can make it extremely costly for large databases. In this paper
we mainly compare our proposed method with the cosine similarity
(\emph{cossim}) label prediction approach. However, our reasoning still
holds for conventional classification due to $\mathcal{O}(n\textsubscript{classes})$
complexity of the final matrix
multiplication operation.

In this paper we consider the task of finding the closest prototype
vector to the input vector (query). In context of LSC this also allows
to predict labels, since prototype vectors represent center vectors for
class embedding clusters~\cite{lsc2}. Finding the closest prototype
vector requires comparing the query vector with
\emph{n\textsubscript{classes}} prototype vectors. However, with the
predefined internal structure of vector system, this task can be
significantly simplified by taking advantage of vector system
properties.

It has previously been noted that $\mathcal{O}(1)$ complexity search for
closest center vectors is possible when root systems are used as targets
for LSC~\cite{Igor}. In this paper we propose simple and computationally
efficient search procedure that can be used upon specific family of
vector systems, as shown in Algorithm~\ref{alg:1}. These vector systems are
obtained as permutations of elements of base vectors with coordinates
consisting of ones, negative ones, and zeroes. It should be stressed
that for Algorithm~\ref{alg:1} to be applicable, NN should be trained with angular
metric (\emph{e.g.}, cosine loss), and its LS has to be configured with
\(V_{n}^{D}\) discussed below as a target, \emph{i.e.} training should
be performed using the methodology proposed in~\cite{lsc2}.

At heart of the proposed approach is the fact that when all center
vectors have exactly \emph{m} ones and \emph{k} negative ones (so
\(V_{n}^{D} = V_{n}^{mk}\)), one can find the closest center vector by
finding the indexes of \emph{m} largest and \emph{k} lowest values in
query's embedding. A center vector with ones in indexes of maximums and
negative ones in indexes of minimums will be the closest possible to
query's embedding. Thus, the closest center search complexity is reduced
to $\mathcal{O}(n\textsubscript{dim})$ requiring only two search
operations, which can be performed for batches of inputs using, for
instance, `topk' function in PyTorch~\cite{topk}. Since for all NN use
cases \emph{n\textsubscript{dim}} is a constant during inference, the
$\mathcal{O}(1)$ complexity search is achieved by taking advantage of the
known \(V_{n}^{D}\) structure used for NN training.

The proposed structure of \emph{center\_dict} in Algorithm~\ref{alg:1} ensures
$\mathcal{O}(1)$ hash map search for fast inference. However, this
implementation is only possible when \emph{center\_dict} can be stored in RAM,
and chunking may be required when data structures become too large. It
should also be noted that while the proposed method is not equivalent to
\emph{cossim}, its closest center prediction always coincides with the
\emph{cossim} prediction. This is proven by the following theorem.

\begin{theorem}
Let $V^{mk}_n$ be the set of all vectors in $\mathbb{R}^n$ having exactly $m$ entries equal to $1$, $k$ entries equal to $-1$, and $n-m-k \ge 0$ zero entries.
Let $W^{mk}_n$ be the set of all vectors $x \in \mathbb{R}^n$ such that, after rearranging the coordinates in the non-decreasing order,
\[
x_1 \le \dots \le x_n,
\]
one has
\[
x_k < x_{k+1}
\quad \text{and} \quad
x_{n-m} < x_{n-m+1}.
\]

Define a mapping $f : W^{mk}_n \to V^{mk}_n$ by assigning to each $w$ the vector $f(w)$ is obtained by placing $1$ at the positions of the $m$ largest coordinates of $w$, placing $-1$ at the positions of the $k$ smallest coordinates of $w$, and placing $0$ at all remaining positions. 

Let $w \in W^{mk}_n$,
\[
  \cossim(v,w) = \frac{\langle v,w \rangle}{\|v\|\|w\|}.
\]

Then
\[
\argmax_{v \in V^{mk}_n} \cossim (v,w) = f(w),
\]

\end{theorem}

\begin{proof}
Let $x_1,\dots,x_m$ be the $m$ largest coordinates of $w$, arranged in the nonincreasing order, and let $y_1,\dots,y_k$ be the $k$ smallest coordinates of $w$, arranged in the non-decreasing order.

Let $v=f(w)$. Observe that for every $v \in V^{mk}_n$ one has
\[
\|v\| = \sqrt{m+k}.
\]

Suppose that there exists $v' \in V^{mk}_n$ such that
\[
\cossim(v,w) < \cossim(v',w).
\]
Since $\|v\|=\|v'\|$, this implies
\[
\langle v,w\rangle < \langle v',w\rangle.
\]

The scalar product $\langle v',w\rangle$ is the sum of $m$ coordinates of $w$ taken with the sign $+$ and $k$ coordinates taken with the sign $-$. Thus
\[
\langle v',w\rangle = z_1+ \dots + z_m - t_1- \dots - t_k,
\]
where $z_i$ are some $m$ coordinates of $w$ arranged in the nonincreasing order and $t_i$ are some $k$ coordinates of $w$ arranged in the nondecreasing order.

Hence
\[
x_1 + \dots + x_m - y_1- \dots - y_k 
<
z_1 + \dots + z_m - t_1 - \dots - t_k,
\]
which implies
\[
\sum_{i=1}^m (x_i - z_i)
+
\sum_{j=1}^k (t_j - y_j)
< 0.
\]

However, each term in this sum is nonnegative by construction. This contradiction completes the proof.
\end{proof}

\subsection{Handling unlabeled center vectors}
\label{unlabeled_centers}

As mentioned in Section~\ref{Vn}, the number of vectors
\emph{n\textsubscript{vects}} in \(V_{n}^{D}\) is determined by
\emph{n\textsubscript{dim}} and \emph{D}. Then
\emph{n\textsubscript{classes}} vectors are chosen as targets for class
center vectors during NN training (hereafter called labeled centers), so
\emph{n\textsubscript{classes}} \(\leq\) \emph{n\textsubscript{vects}}.
Thus, in most cases there are several vectors present in \(V_{n}^{D}\)
which do not correspond to any class (hereafter called unlabeled
centers)

\begin{equation}
	n_{uc} = n_{vects} - n_{classes},
	\label{eq:n_uc}
\end{equation}
\unskip

\begin{equation}
	k_{l} = \frac{n_{classes}}{n_{vects}},
	\label{eq:k_l}
\end{equation}
\unskip

The proposed approach returns the closest possible vector in
\(V_{n}^{D}\), whether it is labeled or not. Returning unlabeled centers
is only possible when the label coefficient (\ref{eq:k_l}) is less than one.
One can see that \emph{k\textsubscript{l}} is the closest to one when
the minimum possible \emph{n\textsubscript{dim}} for desired
\emph{n\textsubscript{classes}} is chosen using (\ref{eq:n_vects}). However, when
\emph{n\textsubscript{dim}} is large relative to
\emph{n\textsubscript{classes}}, \emph{k\textsubscript{l}} gets smaller,
indicating that more unlabeled centers are present in LS. While a
special case of the unlabeled centers usage is discussed in Section~\ref{discussions},
Algorithm~\ref{alg:unlabeled} provides a procedure for finding the closest labeled center
to return its label in the conventional label prediction setting.

Let $U \subset V^{mk}_n$. To find $\argmax_{v \in U} \cossim(v,w)$,
where $w$ is such that the $\argmax$ function is well-defined (i.e., the maximum element is unique), one can use the following approach.

Let $x = (x_1, \dots, x_n)$ be the vector obtained from $w$ by arranging its coordinates such that $x_i \le x_{i+1}$. Then $\sigma w = x$, where $\sigma$ is the permutation that sorts the coordinates accordingly.

For each vector $v \in V^{mk}_n$, we define a basic permutation of its coordinates to be a permutation that either swaps -1 with the preceding entry (if it exists) or swaps 1 with the following entry (if it exists), provided that the resulting vector differs from the original one.

Let $G$ be directed graph without loops whose vertex set is the set $V^{mk}_n$. There is directed edge from $v_1$ to $v_2$ if and only if $v_2$ is obtained from $v_1$ by basic permutation of coordinates.

Define a vector $x' = (x'_1, \dots, x'_n) \in \mathbb{R}^n$ by
\[
  x'_i = x_i - x_{i+1}.
\]
Note that $x'_i \le 0$.

We now define weighted graph $G_x$, obtained from $G$ by assigning weights to its edges as follows. If $v_1$ and $v_2$ differ by swapping the coordinates with indices $i$ and $i+1$, then the weight of the edge equals $x'_i = \langle x, v_1 - v_2 \rangle$ in the case of a swap of 1 or -1 with 0, and equals $2x'_i = \langle x, v_1 - v_2 \rangle$ in the case of a swap of -1 and 1.

\begin{lemma}
  The weighted directed graph $G_x$ posesses the following properties:
  \begin{enumerate}
  \item All edge weights are nonnegative.
  \item Every vertex $v \neq (-1, \dots, -1, 0, \dots, 0, 1, \dots, 1)$ has at least one outgoing edge.
    \item For any vertices $v_1$ and $v_2$, the sum of weights of any directed path from $v_1$ to $v_2$ is independent of the choice of the path and equal to $\langle x, v_1 - v_2 \rangle$.
  \end{enumerate}
\end{lemma}
\begin{proof}
  \begin{enumerate}
  \item This follows immediately from the fact that $x'_i \le 0$ for all $i$.
  \item Indeed, any vertex different from $(-1, \dots, -1, 0, \dots, 0, 1, \dots, 1)$ admits either a shift of at least one occurrence of -1 to the left or of at least one occurrence of 1 to the right, which produces an outgoing edge.
    \item This follows from the definition of the edge weights together with the linearity of the scalar product in the second argument.
    \end{enumerate}
 \end{proof}

 Thus, the problem of finding a vector $\argmax_{v \in U} \cossim(v, w)$ reduces to finding
 \[
   v^* = \sigma^{-1} \hat{v}^*,
 \]
 where $\hat{v}^* \in \sigma U$ is a vertex closest to $(1, \dots, 1, 0, \dots, 0, -1, \dots, -1)$ with respect to the path metric on $G_{\sigma w}$, that is, minimizing the total weight of a path connecting these vertices.

 Based on this lemma, the search for the optimal vector in a constrained subset $U$ can be formulated as follows.

 \begin{algorithm}[H]
   \caption{Finding the closest labeled center.}
   \label{alg:unlabeled}
   \begin{algorithmic}[1]
     \Require Vector $w \in \mathbb{R}^n$, constraint subset $U \subset V^{mk}_n$.
     \Ensure $v^* = \argmax_{v \in U} \cossim(v,w)$.
     \State $\sigma \gets$ permutation such that coordinates of $\sigma w = x$ are non-decreasing.
     \State $\hat{U} \gets \sigma U$.
     \State $S \gets \varnothing$. \Comment{set of candidate vertices}
     \State $v \gets (1, \dots, 1, 0, \dots, 0, -1, \dots, -1)$.

     \Procedure{DFS}{current, $S$}
     \If{current $\in \hat{U}$ or current $= -(1, \dots, -1, 0, \dots, 0, 1, \dots, 1)$}
     \State $S \gets S \cup \{ \text{current} \}$.
     \Else
     \For{vertex $\gets$ each outgoing edge from current in $G$}
     \State DFS(vertex, S)
     \EndFor
     \EndIf
     \EndProcedure
     
     \State DFS($v$, $S$)
     \State $\hat{v}^* \gets \argmax_{v \in S} \cossim(v, w)$.
     \State $v^* \gets \sigma^{-1} \hat{v}^*$.
   \end{algorithmic}
  \end{algorithm}

\subsection{Comparison with the advanced similarity search algorithms}
\label{advcossim}

Cosine similarity search is the primary method for similarity search in
NN applications and databases with its $\mathcal{O}(n)$ complexity being its
well-known drawback. Various methods which can significantly reduce its
complexity have been proposed over last decades. These methods usually
involve combining cosine similarity search with other techniques which
include database clustering, an approximate \emph{k}-means nearest
neighbor search, vector quantization techniques, etc.~\cite{pLHC,BSG,FSS,PQ}. 
Unfortunately, most of them trade accuracy for speed while still
being very computationally intensive. Furthermore, these algorithms
still have $\mathcal{O}(n\textsuperscript{1/b})$ complexity, where \emph{b}
is method-dependent, and in most cases \(1 \leq b \leq 2\). Hence, while
indeed being faster than the original similarity search, their
complexity remains proportional to \emph{n}. 

\section{Experiments}
\label{Exp}

\begin{table}
	\caption{The comparison of training accuracy calculation results with
		\emph{cossim} and the proposed method.} 
	\label{tab:31}
	\centering
	\begin{tabular}{|c|c|c|c|c|c|}
	  \hline
	  exp. & \emph{n\textsubscript{dim}} & \emph{n\textsubscript{classes}}
	  & Dataset size & Acc. cossim, \% & Acc. new, \% \\ \hline
		1. & 10 & 1000 & 1281k & 94.5 & 94.5 \\
		2. & 23 & 50k & 50k & 98.3 & 98.3 \\
		3. & 48 & 300k & 300k & 96.2 & 96.2 \\
		\hline
	\end{tabular}.
\end{table}

The experimental procedure in this study follows the one described in
Section 5.3 in~\cite{lsc2}. That is, the datasets used for NN inference
were created from ImageNet-1K (i1k) by assigning a unique class to every
image. The datasets with over 1 million classes were obtained using
multiple image copies with different labels. However, only 1 million
images from each dataset were used in experiments in Section~\ref{exp_speed} to
remove the dependence of computational time on dataset size. All
experiments use the same preprocessing procedure which includes input
image resizing, center crop, conversion to PyTorch tensors, and
application of conventional i1k normalization.

Small visual transformer (ViT-S)~\cite{VT} NN models with \(V_{n}^{22}\)
LS configuration are used for all experiments. The experiments are
conducted using Intel(R) Xeon(R) Gold 6230R CPU and a single 40GB NVIDIA
A100 GPU. All PyTorch CPU operations that can be run in parallel are
limited to two CPU threads. There are two types of experiments: with a
fixed batch size of 256, and the largest batch size that can fit in 40GB of GPU memory.

The main purpose of the conducted experiments is to compare the time
required to obtain the predicted label for query images with
conventional cosine similarity and Algorithm~\ref{alg:1}. In order to avoid
ambiguity regarding whether or not the preprocessing time should be
included in model inference time, we explicitly refer to all preliminary
operations including data loading and preprocessing as
\emph{t\textsubscript{d}}, and to the time required for NN model
computations as forward time \emph{t\textsubscript{f}}. We also refer to
the time required to obtain label from NN embedding with \emph{cossim}
and the new method as \emph{t\textsubscript{c}} and
\emph{t\textsubscript{n}}, respectively. Measured time intervals are
used to compute the search and the total acceleration coefficients as

\begin{equation}
	K_{s} = \frac{t_{c}}{t_{n}},
	\label{eq:K_s}
\end{equation}
\unskip

\begin{equation}
	K_{t} = \frac{t_{d} + t_{f} + t_{c}}{t_{d} + t_{f} + t_{n}},
	\label{eq:K_t}
\end{equation}
\unskip

\subsection{Accuracy comparison of cossim and the proposed method}
\label{exp_acc}

Table~\ref{tab:31} shows training dataset classification accuracy calculation
results on i1k and two datasets with large
\emph{n\textsubscript{classes}}. The accuracy is computed as a ratio of
correct predictions to all predictions (dataset size). Table~\ref{tab:31} shows
the exact agreement between the results obtained with \emph{cossim} and
the new method.

\subsection{Large n\textsubscript{classes} dataset processing speed comparison}
\label{exp_speed}

\begin{table}
	\caption{The time required by different computational stages of
		model inference on 1m images with \emph{cossim} and the proposed method
		with maximum possible batch size for 40GB GPU memory.} 
	\label{tab:321}
	\centering
	\begin{tabular}{|c|c|c|c|c|c|c|c|c|c|}
	  \hline
	  exp. & \emph{n\textsubscript{classes}} & \emph{n\textsubscript{dim}}
	  & Batch size & \emph{t\textsubscript{d}}, s & \emph{t\textsubscript{f}}, s
	  & \emph{t\textsubscript{c}}, s & \emph{t\textsubscript{n}}, s & \emph{K\textsubscript{s}} 
	  & \emph{K\textsubscript{t}} \\ \hline
		1 & 1m & 47 & 3584 & 4536 & 6 & 20 & 1.6 & 12 & 1.003 \\
		2 & 5m & 69 & 2560 & 4439 & 4 & 175 & 2.3 & 76 & 1.04 \\
		3 & 10m & 82 & 1024 & 4325 & 8.6 & 805 & 3.4 & 238 & 1.18 \\
		4 & 20m & 97 & 256 & 4444 & 33 & 7350 & 4 & 1836 & 2.64 \\
		5 & 30m & 107 & 64 & 4334 & 126 & 47244 & 8.2 & 6050 & 11.57 \\
		\hline
	\end{tabular}
\end{table}

\begin{table}
	\caption{The time required by different computational stages of
		model inference on 1m images with \emph{cossim} and the proposed method
		with fixed batch size of 256. Dash symbol indicates that \emph{cossim}
		can no longer be used.} 
	\label{tab:322}
	\centering
	\begin{tabular}{|c|c|c|c|c|c|c|c|c|c|}
	  \hline
	  exp. & \emph{n\textsubscript{classes}} & \emph{n\textsubscript{dim}}
	  & Batch size & \emph{t\textsubscript{d}}, s & \emph{t\textsubscript{f}}, s
	  & \emph{t\textsubscript{c}}, s & \emph{t\textsubscript{n}}, s & \emph{K\textsubscript{s}} 
	  & \emph{K\textsubscript{t}} \\ \hline
		1 & 1m & 47 & 256 & 4411 & 33 & 97 & 3.1 & 31 & 1.02 \\
		2 & 5m & 69 & 256 & 4235 & 33 & 1258 & 3.1 & 403 & 1.29 \\
		3 & 10m & 82 & 256 & 4232 & 33 & 1825 & 3.7 & 485 & 1.43 \\
		4 & 20m & 97 & 256 & 4444 & 33 & 7350 & 4 & 1836 & 2.64 \\
		5 & 30m & 107 & 256 & 4421 & 33 & - & 3.7 & - & - \\
		6 & 50m & 121 & 256 & 4290 & 31 & - & 3.6 & - & - \\
		7 & 100m & 143 & 256 & 4347 & 32.5 & - & 3.8 & - & - \\
		\hline
	\end{tabular}
\end{table}

The main experimental results of this study are summarized in Tables~\ref{tab:321}
and~\ref{tab:322}. Table~\ref{tab:321} shows that the overall speed-up is
negligible when large GPU batches are used even for rather large 1-5
million class datasets. Whereas in these cases \emph{K\textsubscript{s}}
is already considerable, \emph{K\textsubscript{t}} is small because
other operations take significantly more time than search operations.
However, \emph{K\textsubscript{t}} starts to grow as
\emph{n\textsubscript{classes}} increases and batch size reduces due to
GPU memory limitations. For the 20 million classes experiment, the batch
size can no longer be increased over the baseline size of 256, which
being combined with the growing cost of \emph{cossim} operations results
in a 2.64 times speedup when changing to the proposed method.
Furthermore, \emph{K\textsubscript{t}}reaches the astounding value of
11.6 in the 30 million classes experiment when the batch size has to be
further reduced to 64.

Table~\ref{tab:322} shows that with fixed batch size, the advantage of using the
proposed method becomes apparent even at 5 million classes. Furthermore,
it illustrates the possibility of conducting extremely large
\emph{n\textsubscript{classes}} inference since GPU memory requirements
do not depend on \emph{n\textsubscript{classes}}. This allows conducting
experiments in cases beyond the \emph{cossim} capabilities, as shown by
50 and 100 million classes experiments. These observations are
illustrated by Figure~\ref{fig:32}.

It should be noted that in Table~\ref{tab:321} \emph{t\textsubscript{n}}
increases between experiments mainly due to batch size reduction, while
in Table~\ref{tab:322} \emph{t\textsubscript{n }}slightly varies due to changes
in \emph{n\textsubscript{dim}}. Tables~\ref{tab:321} and~\ref{tab:322} also show that in
these experiments the most time-consuming operations are data loading
and preprocessing. Hence, the advantage of using the proposed method
increases in cases where \emph{t\textsubscript{d}} can be decreased by
reducing image size, using faster preprocessing, etc. Equations (\ref{eq:K_s})
and (\ref{eq:K_t}) show that \emph{K\textsubscript{t}} tends towards
\emph{K\textsubscript{s}} as \emph{t\textsubscript{d}} and
\emph{t\textsubscript{f}} become smaller, and \emph{K\textsubscript{s}}
can reach tremendous values, as shown in Tables~\ref{tab:321} and~\ref{tab:322}. This
illustrates significant potential of using the proposed method combined
with fast I/O and small NN models.

\begin{figure}
	\centering
	\includegraphics[scale=0.45]{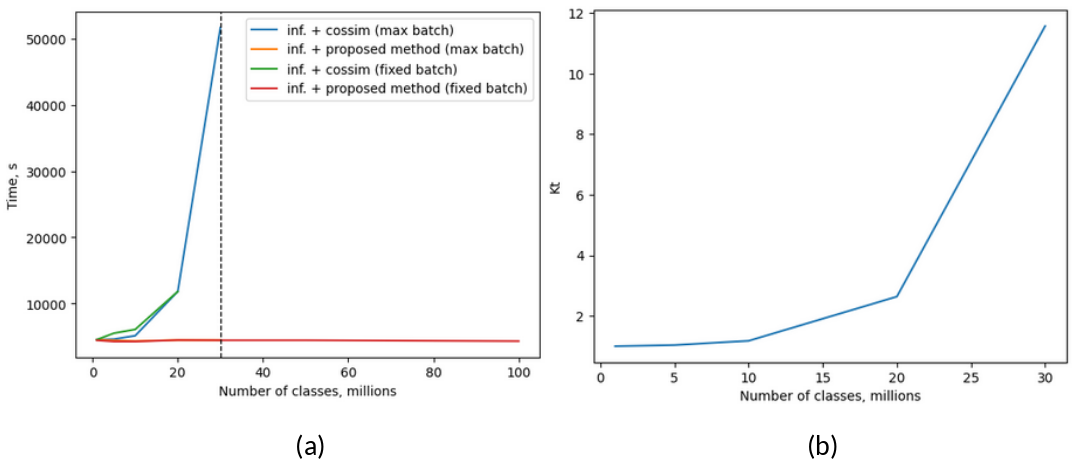} 
	\caption{(a) The total time required to obtain labels for 1m images
		depending on \emph{n\textsubscript{classes}} with \emph{cossim} and the
		proposed method (dashed line indicates where batch size becomes too
		small for \emph{cossim} to be applicable), and (b) the total
		acceleration coefficient \emph{K\textsubscript{t}} as function of
		\emph{n\textsubscript{classes }}(see Table~\ref{tab:321}).}
	\label{fig:32}
\end{figure}
\unskip

\section{Discussions}
\label{discussions}

\subsection{Closest vector search in configurations obtained by
combining multiple vector systems}
\label{combVn}

When NN embeddings are normalized, all possible embedding vectors form
\emph{n}-dimensional hypersphere. However, \(V_{n}^{D}\) vectors are
located on a (\emph{n-1})-dimensional hyperplane within this hypersphere
(see Section 3.2.1 in~\cite{lsc2} and Section 2.1 in~\cite{lsc3} for
details). Therefore, vectors of single \(V_{n}^{D}\) can have a
relatively poor hypersphere occupation. Whereas this is not necessarily
a disadvantage due to the tendency of NNs to form low-dimensional
manifold solutions in \emph{n} dimensions~\cite{iD}, there might be
applications that require a better hypersphere occupation than single
\(V_{n}^{D}\) can provide.

It has previously been proposed in~\cite{lsc2} that in case of
\emph{A\textsubscript{n}}, \emph{n}-dimensional vectors with better
hypersphere occupation can be obtained by projecting
\emph{(n+1)}-dimensional vectors of the original root system back to
\emph{n} dimensions. This adds some degree of non-uniformity to the
vector distribution, which has also been found to be beneficial for NN
training in case of \emph{A\textsubscript{n}} configuration. The same
operation can be applied to any \(V_{n}^{D}\) system. Resulting systems,
denoted \emph{p} for projected, can be expressed as combinations of
original systems

\begin{equation}
	V_{np}^{mk} = \pi_{n}\left( V_{n + 1}^{mk} \right) = V_{n}^{mk} \cup V_{n}^{\acute{m}k} \cup V_{n}^{m\acute{k}},
	\label{eq:V_proj}
\end{equation}
\unskip

where \(\pi_{n}\) is a projection operator which drops the
\emph{(n+1)\textsuperscript{th}} coordinate of all vectors in
\(V_{n + 1}^{mk},\) \(\acute{m} = m - 1\) and \(\acute{k} = k - 1\). One
can see that searching for the closest vector in this case can be done
by separately searching each \(V_{n}^{D}\) and then comparing the
obtained center vectors. Whereas this is longer than performing a search
in single \(V_{n}^{D}\), the combined system search complexity is still
$\mathcal{O}(1)$ since it does not depend on \emph{n\textsubscript{classes}}.
This example also illustrates that different LS configurations can have
different properties, and recommendations for choosing LS configurations
for specific tasks will be formulated in the future.

\subsection{Using unlabeled centers for new class prediction}
\label{newclasspred}

\begin{figure}
	\centering
	\includegraphics[scale=0.5]{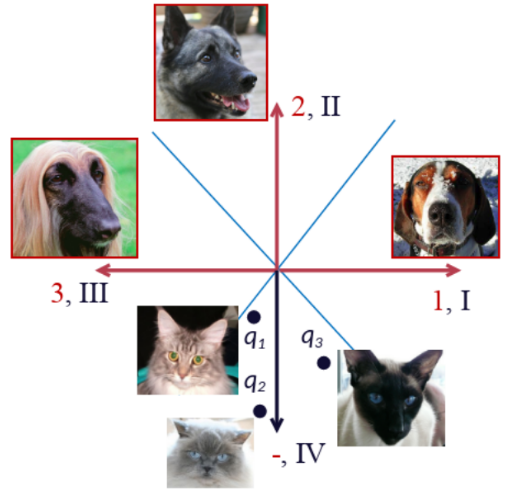} 
	\caption{A ``toy'' example of the dog-only detection system receiving
		cat queries: red arrows represent labeled centers, black arrow
		represents the unlabeled center, red digits represent data labels, black
		roman numerals represent center indexes, dog images are example classes
		1-3, cat images correspond to example queries
		\emph{q\textsubscript{1}}-\emph{q\textsubscript{3}}. Blue lines
		correspond to decision boundaries, or angles where closest centers
		change.}
	\label{fig:42}
\end{figure}
\unskip

It has previously been discussed in Section~\ref{unlabeled_centers} that the proposed method
can return an unlabeled center and how the next closest labeled center
can be found to predict the label. However, the ability to return
unlabeled centers in Algorithm~\ref{alg:1} is a feature which can have other uses.

Finding multiple samples belonging to the same unlabeled center can
imply that there is a class previously unseen by NN during training.
This point is illustrated by a toy example in Figure~\ref{fig:42}. Typical NNs
with classification layers are bound to return a label from
\emph{n\textsubscript{classes}} set for any, however absurd, input. This
is also true for systems which check against
\emph{n\textsubscript{classes}} prototypes. This means that all queries
\emph{q}, while being closer to the unlabeled center IV, will receive
either label 1 of center I, or label 3 of center III. However, the
proposed method can help to identify the anomaly in the stream of inputs
and suggest that such queries are evidence of the presence of a new
class at the unlabeled center IV (which, for this example, are cats).
Hence, this can act as a message for human operator that there is
something unexpected about queries rather than returning classes from a
predefined set of options.

This fact, combined with the possibility to easily change the number of
classes with LSC during training and inference (see Section 5.1.1 in~\cite{lsc2}),
can be applied to areas where the number of classes is
inherently not constant. It was previously suggested that lifelong
learning and few-shot learning were promising application areas for LSC~\cite{lsc2}.
This natural capability to identify new classes by analyzing
inference output reinforces the usefulness of proposed method even
further. However, it should be stressed that the capability to identify
new classes does not guarantee this behavior, and further research must
be conducted to verify these assumptions. This topic will be studied in
greater detail in the future.

\section{Conclusions}
\label{conclusions}

This paper verifies the possibility to conduct $\mathcal{O}(1)$ complexity
closest cluster center search for NN embeddings in a vector system used
for LSC. This in turn leads to very fast and computationally efficient
NN label prediction. Experimental results show the exact agreement in
accuracy between the proposed and \emph{cossim} methods. The time
measurement experiments show that the advantage of using the proposed
method becomes notable when the number of classes exceeds 10 million,
especially when GPU memory is limited. In latter scenario the proposed
method allows to achieve up to 11.6 times acceleration due to the
forcibly reduced batch size and the overall growth in computational cost
of label prediction operation in conventional methods. The proposed
method also allows to conduct inference when the number of classes is
extremely large and beyond the conventional methods' capabilities, as
illustrated by 100 million classes experiment. An additional capability
to return information about presence of unlabeled nearest vector, which
can help to identify new classes, can be very promising for LSC
application in lifelong and few-shot learning.

\section*{Acknowledgement}
\label{acknowledgement}

The authors would like to thank their colleagues Dr Igor Netay and Dr
Anton Raskovalov for fruitful discussions, and Vasily Dolmatov for
discussions and project supervision.


\bibliographystyle{IEEEtran}
\bibliography{IEEEabrv,ms}